\documentclass[journal,twoside]{IEEEtran}

\usepackage{cite}
\usepackage{amsmath,amssymb,amsfonts}
\usepackage{algorithmic}
\usepackage{graphicx}
\usepackage{algorithm,algorithmic}
\usepackage{hyperref}
\hypersetup{hidelinks=true}
\usepackage{textcomp}
\def\BibTeX{{\rm B\kern-.05em{\sc i\kern-.025em b}\kern-.08em
    T\kern-.1667em\lower.7ex\hbox{E}\kern-.125emX}}

\usepackage{subfigure}
\newtheorem{theorem}            {Theorem}[section] 
\newtheorem{definition}         [theorem]{Definition} 
\newtheorem{assumption}         {Assumption}[section] 

\newtheorem{proposition}		[theorem]{Proposition} 
\newtheorem{remark}	      [theorem]{Remark}
\newcommand{\dd}   {{\rm d}\hbox{\hskip 0.5pt}}
\newcommand{\rline}{{\mathbb R}}
\newcommand{\bbm}[1]{\left[\begin{matrix} #1 \end{matrix}\right]}
\newcommand{\sbm}[1]{\left[\begin{smallmatrix} #1
	\end{smallmatrix}\right]}
\newcommand{\degree}{^{\circ}}
\usepackage{bm}

\begin{document}
\title{Collision-free Source Seeking Control Methods for Unicycle Robots} 

\author{Tinghua Li, Bayu Jayawardhana, \IEEEmembership{Senior Member, IEEE}
\thanks{Tinghua Li and Bayu Jayawardhana are with the Engineering and Technology Institute Groningen, Faculty of Science and Engineering, University of Groningen, 9747 AG Groningen, The Netherlands (e-mail: lilytinghua@gmail.com;
b.jayawardhana@rug.nl).
}
}

\maketitle

\begin{abstract}
In this work, we propose a collision-free source-seeking control framework for a unicycle robot traversing an unknown cluttered environment. In this framework, obstacle avoidance is guided by the control barrier functions (CBF) embedded in quadratic programming, and the source-seeking control relies solely on the use of onboard sensors that measure the signal strength of the source. To tackle the mixed relative degree and avoid the undesired position offset for the nonholonomic unicycle model, we propose a novel construction of a control barrier function (CBF) that can directly be integrated with our recent gradient-ascent source-seeking control law. We present a rigorous analysis of the approach. The efficacy of the proposed approach is evaluated via Monte-Carlo simulations, as well as, using a realistic dynamic environment with moving obstacles in Gazebo/ROS. 
\end{abstract}


\IEEEpeerreviewmaketitle

\vspace{-0.1cm}
\section{Introduction}
In the development of autonomous systems, such as autonomous vehicles, autonomous robots and autonomous spacecraft, safety-critical control systems are essential for ensuring the attainment of the control goals while guaranteeing the safe operations of the systems 
\cite{JC2002, Hovakimyan2001}. 
For example, autonomous robots may need to fulfill the source-seeking task while navigating in an unknown environment safely. This capability is important for search and rescue missions and for chemical/nuclear disaster management systems. Unlike standard robotic control systems that are equipped with path-planning algorithms, the design of safe source-seeking control systems is challenging for several factors. Firstly, the source/target location is not known apriori and hence path planning cannot be done beforehand. Secondly, the lack of global information on the obstacles in an unknown environment (such as underwater, indoor or hazardous disaster area) prevents the deployment of safe navigation trajectory generation \cite{Zou2015, Azuma2012}. Thirdly, the control systems must be able to solve these two sub-tasks consistently without generating conflicts of control action.

As another important control problem in high-tech systems and robotics, motion control with guaranteed safety has become essential for safety-critical systems. In this regard, the obstacle avoidance problem in robotics is typically rewritten as a state-constrained control problem, where collision with an obstacle is regarded as a violation of state constraints. Using this framework, multiple strategies have been proposed in the literature that can cope with static and dynamic environments. One popular obstacle avoidance algorithm is inspired by the use of the control barrier function (CBF) for nonlinear systems \cite{Prajna,Wieland, Xiao}. The integration of stabilization and safety control can be done either by combining the use of control Lyapunov function (CLF) with CBF as pursued in \cite{Romdlony2016}, or by recasting the two control objectives in the constraint of a given dynamic programming, such as the ones using quadratic programming (QP) in \cite{Ames_CDC,Ames_TAC}.
A relaxation on the sub-level set condition in the CBF has been studied in \cite{Kong,Dai}. For the approach that recasts the problem into QP, the early works in \cite{Ames_ECC,Glotfelterp} are applicable only to systems with relative-degree of one, where the relative degree is with respect to the barrier function as the output. 

Another common approach of collision-free source-seeking is extremum-seeking-control (ESC) \cite{ref-1,ref-2,ref-3}, where the potentials of the source and the obstacles are integrated into a navigation function such that the steady-state gradient of potential can be estimated. However, given the application of excitation signals for inferring the cost function, this approach can result in slow convergence and is affected by the initialization. Moreover, the repulsive potential of the obstacles is integrated into the control input at all time. 
As far as the authors are aware, the integration of CBF with the aforementioned source-seeking control problem has not yet been reported in the literature. In contrast to the path-planning problem where the target location is known in advance, the combination of the CBF-based method with source-seeking control is non-trivial due to the lack of information on the source location. Moreover, solving this problem for non-holonomic systems (such as unicycles) adds to the challenge, where the nominal distance-based CBF can cause a relative degree problem for the control inputs (i.e., \emph{relative degree of one} with respect to longitudinal velocity input and \emph{relative degree of two} with respect to the angular velocity input). Using existing methods proposed in the literature, it is not trivial to include both inputs in the CBF constraint formulated in QP. A standard solution is to linearize the unicycle by considering a virtual point instead of the robot's center \cite{Glotfelter,Majd,Yaghoubi}, however, it would result in an undesired offset in avoidance motion. Another solution is to achieve both tasks by solely controlling the angular velocity and maintaining a constant small linear velocity, which limits the performance of the robot in motion. 

To overcome the mixed relative degree problem and provide flexibility for unicycle robot motion, we propose the collision-free source-seeking control system for the autonomous robot that is described by the unicycle system. The considered source field can represent the concentration of chemical substance/radiation in the case of chemical tracing, the heat flow in the case of hazardous fire, the flow of air or water in the case of locating a potential source, etc. The field strength is assumed to decay as the distance to the source increases. We extend our previous source-seeking control work in \cite{Li}, where a projected gradient-ascent control law is used to solve the source-seeking problem, by combining it with the control barrier function (CBF) method to avoid obstacles on its path toward the source. In summary, our main contributions are:
\begin{enumerate}
    \item A novel construction of zeroing control barrier function (ZCBF) is proposed for a unicycle robot using an extended state space to get a uniform relative degree with respect to both longitudinal and angular velocity control inputs. 
    \item Only the local sensing measurements are required to achieve both source-seeking (requiring the gradient of the source field) and obstacle avoidance task (requiring the relative bearing and distance to the closest obstacle).
    \item  The asymptotic convergence and the safety properties of the closed-loop systems are analyzed. The proposed methods can be implemented numerically using any existing QP solvers, and their efficacy is verified in Monte Carlo simulation using Matlab and in a realistic environment using Gazebo/ROS platforms, where a smooth trajectory and fast convergence are demonstrated. 
\end{enumerate}
 The rest of the paper is organized as follows. In Section \ref{sec:problem}, we formulate the safety-guaranteed source-seeking control problem. The control integration methods and the constructions of the control barrier functions are illustrated in Section \ref{sec:control_design}.  Section \ref{sec:simulation} presents the efficacy of the proposed methods via Monte Carlo simulations, as well as, using realistic environment in Gazebo/ROS. The conclusions and future work are provided in Section \ref{sec:conclusion}. 

\section{Problem Formulation and Preliminaries} \label{sec:problem}
{\bf Notations.} For a vector field $f:\rline^n\to\rline^n$ and a scalar function $h:\rline^n\to\rline$, the Lie derivative of $h$ along the vector field $f$ is denoted by $L_fh(x):\rline^n\to\rline$ and defined by $L_fh(x)=\frac{\partial h(x)}{\partial x}f(x)$. A continuous function $\alpha : (-b,a) \rightarrow (-\infty, \infty)$ is said to belong to \emph{extended class} $\mathcal{K}_e$ for some $a,b>0$ if  it is strictly increasing and $\alpha(0)=0$.

\vspace{-0.2cm}
\subsection{Problem Formulation } \label{sec:Problem}
For describing our safety-guaranteed source-seeking problem in a 2-D plane, let us consider a scenario where the mobile robot has to locate an unknown signal source while safely traversing across a cluttered environment. The signal strength decays with the increasing distance to the source, and the robot is tasked to search and approach the source as quickly as possible while performing active maneuvers to avoid collision with the obstacles based on the available local measurements.
\begin{assumption}\label{ass:J}
   \rm{The source distribution $J(x,y)$ of the obstacle-occupied environment is a twice differentiable, radially unbounded strictly concave function in the $(x,y)$-plane with a global maximum at the source location $(x^*,y^*)$, e.g. $J(x^*, y^*)>J(x,y)$ for all $(x,y)\neq (x^*,y^*)$. }
 \end{assumption}  
\begin{assumption}
    \rm{The source is located in the safe region and the local source field signal can be measured at any free location in the cluttered environment. Furthermore, we assume that the minimum distance between obstacles is given by $d_{\text{min}}>0$. } 
\end{assumption}
\begin{remark}
    \rm{The robot is equipped with an on-board sensor system and it is only able to obtain the local measurements, including the source signal, and the distance \& bearing with respect to the environmental obstacles. The signal distribution function $J(x,y)$ and the positions of obstacles are both unknown apriori. }
\end{remark}
 
Consider a unicycle robot that is equipped with local sensors for measuring the source gradient and the distance to obstacles in the vicinity. Its dynamics is given by
\begin{equation}\label{eq:unicycle_model}
\bbm{
\dot{x} \\
\dot{y}  \\
\dot{\theta}}=
\bbm{\cos(\theta) & 0\\
\sin(\theta) & 0 \\
0 & 1}\bbm{v \\ \omega},
\end{equation}
where $\sbm{x(t) & y(t)}^\top$ is the 2D planar robot's position with respect to a global frame of reference and $\theta(t)$ is the heading angle. As usual, the control inputs are the longitudinal velocity input variable $v(t)$ and the angular velocity input variable $\omega(t)$. 

Let us now present a general safety-guaranteed problem for general affine nonlinear systems given by 
\begin{equation}\label{eq:dynamics_cbf}
   \dot{\xi}=f(\xi)+g(\xi)u, \quad \xi(0)=\xi_0 \in \mathcal{X}_0 
\end{equation}
where $\xi \in {\mathcal{X}} \subset \mathbb{R}^n$ is the state, $u \in \mathbb{R}^m$ is the control input, $f(\xi):\mathbb{R}^n \rightarrow \mathbb{R}^n$ and $g(\xi):\mathbb{R}^n \rightarrow \mathbb{R}^{n\times m}$ are assumed to be locally Lipschitz continuous. It is assumed that the state space $\mathcal X$ can be decomposed into a safe set $\mathcal{X}_s$ and unsafe set $\mathcal{X}_u$, such that $\mathcal{X}_s \cup  \mathcal{X}_u={\mathcal{X}} \subset \mathbb{R}^2$. Furthermore, we assume that the safe set $\mathcal X_s$ can be characterized by a continuously differentiable function $h:\mathbb{R}^n \rightarrow \mathbb{R}$ so that 
\begin{align} \label{eq:safeset0}
    \mathcal{X}_s & = \left \{ \xi  \in {\mathcal{X} \subset}\mathbb{R}^n : h(\xi )\geq 0\right \} \\
\label{eq:safeset1}
    \partial \mathcal{X}_s &= \left \{\xi  \in {\mathcal{X} \subset}\mathbb{R}^n : h(\xi ) = 0 \right \} \\
\label{eq:safeset2}
    \text{Int}(\mathcal{X}_s) & = \left \{\xi  \in {\mathcal{X} \subset}\mathbb{R}^n : h(\xi ) > 0 \right \}
\end{align}
hold where $\partial \mathcal{X}_s$ and $\text{Int}(\mathcal{X}_s)$ define the boundary and interior set, respectively. 
\begin{definition}
(Forward invariant) \rm{The set $\mathcal{X}_s$ is called \emph{forward invariant} for \eqref{eq:dynamics_cbf} if the implication $\xi _0 \in \mathcal{X}_s \Rightarrow \xi (t) \in \mathcal{X}_s$ for all $t$ holds. Then the system \eqref{eq:dynamics_cbf} is safe if $\mathcal{X}_s $ is forward invariant. }
\end{definition}
\begin{definition}\label{def:relative_degree}
     (Relative degree,\cite{Khalil})\,\rm{A sufficiently smooth function $h:\mathbb{R}^n\rightarrow \mathbb{R}$ is said to have \textit{relative degree} $1 \leq \rho \leq n$ with respect to the system \eqref{eq:dynamics_cbf} in a region $\mathcal{R} \subset \mathcal{X}$ if 
    \begin{enumerate}
       \item $L_gL^{i-1}_fh(\xi) = \bm{0}_{1\times m}\,\, \forall 1\leq i\leq  \rho-1$;
        \item $L_gL^{\rho-1}_fh(\xi) \neq \bm{0}_{1\times m} \,\,  \forall \xi\in\mathcal{R}$.
    \end{enumerate}}
\end{definition}
For simplicity of formulation and presentation, we assume that there is a finite number of obstacles, labeled by $\{\mathcal O_1,\mathcal O_2,\ldots,\mathcal O_q\}$, where $q$ is the number of obstacles and each obstacle $\mathcal O_i$ is an open bounded set in $\rline^2$. 
\begin{remark}
    Throughout the paper, we analyze the safe interaction of unicycle robot regarding the fixed and static obstacles $\mathcal{O}$ in the cluttered environment. 
\end{remark}
\emph{Safety-guaranteed source seeking control problem:} Given the unicycle robotic system \eqref{eq:unicycle_model} with the initial condition $\sbm{x_0 & y_0 & \theta_0}^\top \in \mathcal{X}_0$ and with a given set of safe states $\mathcal{X}_s := \Omega \times \rline$ where $\Omega\subset \mathbb{R}^2\backslash \left(\cup_i\mathcal O_i\right)$ is the set of safe states in the 2D plane, design a feedback control law $v^*, \omega^*$, such that 
\begin{equation}\label{eq:convergence-0}
\lim_{t\to\infty} \left\| \bbm{x(t)-x^*\\y(t)-y^*} 
\right\| = 0, 
\end{equation}
and the unicycle system is safe at all time, i.e., 
\begin{equation}\label{eq:safe_goal}
\setlength{\abovedisplayskip}{5pt}
\setlength{\belowdisplayskip}{5pt}
\sbm{x(t)\\y(t)\\ \theta(t)} \in \mathcal{X}_s, \qquad \forall t\geq 0.
\end{equation}

\vspace{-0.2cm}
\subsection{Source-seeking Control with Unicycle Robot}\label{sec:source_seeking}
In our recent work \cite{Li}, a source-seeking control algorithm is proposed for a unicycle robot using the gradient-ascent approach. We demonstrated that the proposed controller can steer the robot towards a source, e.g. \eqref{eq:convergence-0} holds for all initial conditions $(x(0), y(0),\theta(0)) \in \mathcal X_0$, based solely on the available gradient measurement of the source field and robot's orientation. Specifically, the control law generates longitudinal velocity  $ {v_s}=F(\nabla J(x,y),\theta)$ and angular velocity $\omega_s=G(\nabla J(x,y),\theta)$ as
\begin{equation}\label{eq:SS}
{u_s = \begin{bmatrix}
 v_s\\ 
 \omega_s
\end{bmatrix}}=\begin{bmatrix}
 k_1 \left\langle \vec{o}(\theta),\nabla J(x,y) \right\rangle  \\ 
 -k_2 \left\langle \vec{o}(\theta),{\nabla J^\perp}(x,y) \right\rangle
\end{bmatrix},
\end{equation}
where $\vec{o}(\theta)=\sbm{\cos(\theta) &
\sin(\theta)}$ is the robot's unit vector orientation, the variable $\nabla J(x,y)$ is the source's field gradient measurement, e.g., $\nabla J(x,y)=\sbm{\frac{\partial J}{\partial x}(x,y) & \frac{\partial J}{\partial y}(x,y) }$ with $\nabla J^\perp$ denoting the orthogonal unit vector of $\nabla J$, the control parameters are set as $k_1>0, k_2>0$ in the concave source field. Note that the gradient measurement $\nabla J(x,y)$ can be obtained from the local sensor system on the robot, hence apriori knowledge of the source field is not necessary in practice. We refer to our previous work \cite[Section IV-B]{Li} on the local practical implementation of such source-seeking control. 

\vspace{-0.3cm}
\section{Control design and analysis} \label{sec:control_design}
The collision-free source-seeking problem aims to design a control law using local measurement such that the unicycle robot states remain within the safe set $\mathcal{X}_s$ for all positive time (or, equivalently $\mathcal{X}_s$ is forward invariant, \eqref{eq:safe_goal} holds) and asymptotically converge to the unknown source's location (i.e. \eqref{eq:convergence-0} holds). In this section, we present the CBF-based control design that integrates our source-seeking controller \eqref{eq:SS} with the safety constraints, and mainly discuss the construction of CBF that tackles the mixed relative degree problem.
\begin{definition}
    (Control barrier function) \rm{For the dynamical system \eqref{eq:dynamics_cbf}, a $C^1$  
{\em Control Barrier Function} (CBF) is a non-negative function $h:\rline^n\to\rline$ satisfying \eqref{eq:safeset0}-\eqref{eq:safeset2}, and $h(\xi(t))$ remains positive along the trajectories of the closed-loop system (for a given control law $u=k(\xi)$) for all positive time.
In particular, the Zeroing Control Barrier Function (ZCBF) $h$ has to satisfy 
\begin{equation}\label{eq:ZCBF_pro}
    \sup_{u \in \mathcal{U}}[L_fh(\xi )+L_gh(\xi )u+\alpha(h(\xi ))]\geq 0, 
\end{equation}
for all $\xi  \in \mathcal{X}$, where $\mathcal{X}_s \subseteq \mathcal{X}\subset \mathbb{R}^n$ and $\alpha \in \mathcal K_e$ \cite[Def.~5]{Ames_TAC}.}
\end{definition}

Using the notations of a safe set  \eqref{eq:safeset0}-\eqref{eq:safeset2} that will be incorporated in $h$, the following safe set is defined for the safety-guaranteed navigation problem of our mobile robot while seeking the source: 
\begin{equation}\label{eq:safeset}
    \mathcal{X}_s = \left\{\sbm{x \\ y \\ \theta}\in \rline^3\, | \, \text{dist}\left(\sbm{x\\y},\mathcal O_i\right) - d_{\text{safe}} > 0, i=1,\ldots,q \right\}
\end{equation}
where dist$\left(\sbm{x\\y},\mathcal O_i\right)$ is the Euclidean distance of the robot position $\sbm{x\\y}$ to the obstacle $\mathcal O_i$ and $d_{\text{safe}}\in\mathbb{R}_+$ is a prescribed safe distance margin around the obstacle. We can denote by $d_{\text{ro}} =  \text{dist}\left(\sbm{x\\y},\mathcal O_i\right) - d_{\text{safe}}$. In other words, the safe sets are the domain outside the ball of radius $d_{\text{safe}}$ around the obstacles $\mathcal O_i$. Here, we do not prescribe any specific form of the obstacle set $\mathcal O_i$, as the robot's safety only concerns the Euclidean distance between the obstacle's surface and its position. 

The first challenge is the relative degree problem when the source-seeking control law \eqref{eq:SS} is integrated with the control barrier function to guarantee safe passage. The standard use of $h(x,y)$ that depends only on the distance measurement is not suitable for the unicycle-model robot \eqref{eq:unicycle_model} as it will result in a mixed relative degree system. In this case, the condition \eqref{eq:ZCBF_pro} leads to the situation where angular velocity input $\omega$ cannot be used to influence $h(\xi)$ since $\frac{\partial h}{\partial \xi}\sbm{0\\0\\1}=0$. 

The general solution is to apply a virtual leading point in front of the robot with the near-identity diffeomorphism \cite{Glotfelter,Majd,Yaghoubi}, which would result in an undesired offset with respect to the robot's actual position. To avoid the offset limitation, we tackle the mixed relative degree in the unicycle model \eqref{eq:unicycle_model} by considering an extended state space $\xi=\sbm{x & y & v & \dot x & \dot y}^\top$ such that the extended dynamics is given by 
\begin{equation}\label{eq:dynamics}
\setlength{\abovedisplayskip}{5pt}
\setlength{\belowdisplayskip}{5pt}
    \begin{aligned}
        \dot{\xi }= \sbm{
  \dot{x}\\ 
  \dot{y}\\ 
  \dot{v}\\
  \ddot{x} \\
  \ddot{y}
}
&=\sbm{
  v\cos(\theta)\\ 
  v\sin(\theta)\\ 
  0\\ 
  0\\ 
  0
}+\sbm{
  0 & 0 \\ 
  0 & 0 \\ 
  1 & 0 \\ 
  \cos(\theta) & -v\sin(\theta)\\ 
  \sin(\theta) & v\cos(\theta)
}\sbm{
  a\\ 
  \omega
} 
\\& = \underbrace{ \sbm{
  \xi_4\\ 
 \xi_5\\ 
  0\\ 
  0\\ 
  0
}}_{f(\xi )}+\underbrace{\sbm{
  0 & 0 \\ 
  0 & 0 \\ 
  1 & 0 \\ 
  \frac{\xi_4}{\xi_3}  & -\xi_5\\ 
  \frac{\xi_5}{\xi_3} & \xi_4
}}_{g(\xi )}\underbrace{ \sbm{
  a\\ 
  \omega
}}_{u}
    \end{aligned}
\end{equation}
where $u =\sbm{ a \\ \omega} \in \mathbb{R}^2$ is the new control input with the longitudinal acceleration $a$ and the angular velocity $\omega$. Since we consider only the forward direction of the unicycle to solve the safety-guaranteed source seeking control problem, the state space of $\xi$ is given by $\Xi\subset \rline^2\times (0,\infty) \times \rline^2$ where the longitudinal velocity $\xi_3$ is defined on positive real. It follows that the function $f$ is globally Lipschitz in $\Xi$ and $g$ is locally Lipschitz in $\Xi$. Additionally, for this extended state space, the set $\mathcal X_s$ in \eqref{eq:safeset} can be extended into  
\begin{equation}\label{eq:safeset_extended}
    \mathcal{X}_{s,\text{ext}} = \left\{\xi \in \Xi \, \left|  \,  \text{dist}\left(\sbm{\xi_1\\\xi_2},\mathcal O_i\right) \right. - d_{\text{safe}} > 0, i=1,\ldots,q \right\}
\end{equation}

Using the extended dynamics \eqref{eq:dynamics}, a zeroing control barrier function $h(\xi)$ can be proposed which has a uniform relative degree of $1$. In particular, it is described by the new state variables $\xi$ as follows
\begin{equation}\label{eq:zcbf}
    h(x,y,v,\dot{x},\dot{y}) = D(x,y)  e^{-P(x,y,v,\dot{x},\dot{y})}  
\end{equation}
where $D$ is a smooth function such that: 
1). $D(x,y)=0$ when $d_{\text{ro}}:= \text{dist}\left(\sbm{x\\y},\mathcal O_i\right) - d_{\text{safe}} = 0$ (i.e., the distance of robot $\sbm{x\\y}$ to its closest obstacle reaches the safe margin $d_{\text{safe}}$); 2). $D(x,y)<0$ when $d_{\text{ro}}<0$; 3). $D(x,y)= c>0$ when $d_{\text{ro}}\geq d_{\text{cons}}$. Note that for well-posedness, we set $d_{\text{cons}}:=  \frac{d_{\text{min}}}{2}-d_{\text{safe}}$, where $d_{\text{min}}$ is the minimum separation distance between obstacles;
and otherwise 4). $D(x,y)$ is non-decreasing as a function of distance to the closest obstacle.
The function $P$ in \eqref{eq:zcbf} is defined by
\begin{equation*}
\setlength{\abovedisplayskip}{5pt}
\setlength{\belowdisplayskip}{5pt}
\begin{aligned}
      P(x,y,v,\dot{x},\dot{y}) & = \left \langle \vec{o}_{r}, \vec{o}_{ro}\right \rangle+v\delta,
\end{aligned}
\end{equation*}
and $\left \langle \cdot,\cdot \right \rangle$ denotes the usual inner-product operation, $\vec{o}_{r} = \sbm{ \frac{\dot{x}}{v} & \frac{\dot{y}}{v}}$ gives the unit orientation vector of robot, and $\vec{o}_{ro}$ is the unit bearing vector pointing to the closest point of the nearest obstacle as
\begin{equation}\label{eq:oro}
\setlength{\abovedisplayskip}{5pt}
\setlength{\belowdisplayskip}{5pt}
    \vec{o}_{ro} = \frac{\sbm{
    x_{\text{obs},i}-x\\ 
    y_{\text{obs},i}-y}^\top}{
   \text{dist}\left(\sbm{x\\y},\mathcal O_i\right)} = \frac{\sbm{
    x_{\text{obs},i}-x\\ 
    y_{\text{obs},i}-y}^\top}{ \left\|\sbm{x_{\text{obs},i} \\ y_{\text{obs},i}}-\sbm{x\\y}\right\|},
\end{equation}
where $(x_{\text{obs},i}, y_{\text{obs},i})$ denotes the point on the boundary of $\mathcal O_i$ that is the closest to $(x,y)$. 

Hence, the scalar function $D(x,y)$ is a smooth function of the distance to the closest obstacle with an added safe distance margin of $d_{\text{safe}}\in \mathbb{R}_+$ and is tapered off to a constant $c$ when it is far enough from the obstacles. The function $P$ denotes the projection of the robot's orientation $\vec{o}_r$ onto the bearing vector $\Vec{o}_{ro}$, and the sufficiently small $\delta \in \mathbb{R}_+$ is a directional offset constant to ensure that $\left \|L_gh\right\|\neq0$ within the interior of the safe set $\mathcal{X}_{s,\text{ext}}$. Accordingly, as the robot gets closer to any obstacle, both scalar functions $D(x,y)$ and $e^{-P(x,y,v,\dot{x},\dot{y})}$ become smaller. It is equal to zero whenever $D=0$, e.g., $\text{dist}\left(\sbm{x\\y},\mathcal O_i\right)=d_{\text{safe}}$.

\begin{proposition}\label{pro:relative_degree}
    Consider the unicycle model with the extended dynamics \eqref{eq:dynamics} defined in $\mathcal{X}_{s,\text{ext}}$. Then the relative degree of the proposed ZCBF \eqref{eq:zcbf} with respect to the control input variable $a$ and $\omega$ is $1$ uniformly.
\end{proposition}
\begin{proof}
By applying the distance function in the ZCBF \eqref{eq:zcbf} and consider the new dynamics $\dot \xi = f(\xi)+g(\xi)u$ as in \eqref{eq:dynamics}, the Lie derivative of $h(\xi)$ along the vector field $g(\xi)$ can be calculated by
     $L_gh(\xi)=\frac{\partial h}{\partial \xi}g(\xi)= \frac{\partial D}{\partial \xi}g(\xi)e^{-P} + \frac{\partial (e^{-P})}{\partial \xi}g(\xi)D$.
However, as $D(x,y)$ is a function of state $\xi_1:=x$ and $\xi_2:=y$ in \eqref{eq:zcbf}, with $g(\xi)$ as in \eqref{eq:dynamics}, it is straightforward that $\frac{\partial D}{\partial \xi}g(\xi) = \bm{0}_{1\times2}$ and hereby $L_gh(\xi) = \frac{\partial (e^{-P})}{\partial \xi}g(\xi)D$.
Accordingly, it can be expressed as
\begin{equation}\label{eq:Lgh}
\setlength{\abovedisplayskip}{2pt}
\setlength{\belowdisplayskip}{2pt}
        L_gh(\xi)  = \frac{\partial (e^{-P})}{\partial \xi}D
        = \begin{bmatrix}
-D\delta e^{-(p_o+v\delta)} &  D {{p}'_o}e^{-(p_o+v\delta)}
\end{bmatrix}  
\end{equation}  
 where $ p_o =  \left \langle  \vec{o}_r, \vec{o}_{ro} \right \rangle$ describes the projection of the robot orientation $\vec{o}_r$ onto the unit bearing vector $\vec{o}_{ro}$ as in \eqref{eq:zcbf} and \eqref{eq:oro}, and the notation ${p}'_o= \left \langle  \vec{o}_r, \vec{o}^\perp_{ro} \right \rangle$ denotes the projection onto the orthogonal bearing vector $\vec{o}^\perp_{ro}$. It is straightforward that the $ L_gh(\xi) = \bm{0}_{1\times 2}$ holds only when $D=0$. As given in \eqref{eq:zcbf}, $D(x,y)$ is a smooth function with $D(x,y)=0$ iff $d_{\text{ro}}=0$. In other words, $\left \| L_gh(\xi) \right \| \neq 0 $ for all $\xi \in { \mathcal{X}_{s,\text{ext}}}$. 
\end{proof}

We can now combine the source seeking control law as in \eqref{eq:SS} with the ZCBF-based safety constraint by solving the following quadratic programming (QP) problem: 
\begin{align}\label{eq:QP-zcbf}
    u^* & =\mathop{\mathrm{argmin}}\limits_{u \in \mathcal{U}} \frac{1}{2}\left \| u - { u_{\text{ref}}} \right \|^2 
\\  \label{eq:constraint_ZCBF_1}
  \text{s.t. }\quad & L_fh(\xi )+L_gh(\xi )u+\alpha(h(\xi ))\geq 0 
\end{align}
where $u_{\text{ref}}:= \sbm{a_s \\ \omega_s}$ is the reference signal for source-seeking and it is derived from $u_s=\sbm{v_s\\ \omega_s}$ in \eqref{eq:SS}. Since the right-hand side of 
\eqref{eq:SS} is continuously differentiable, the reference acceleration signal $a_s$ can be derived straightforwardly by differentiating $v_s$. In practice, this can be numerically computed.

The resulting control law $u^*$ ensures that it stays close to the reference source-seeking law while avoiding obstacles through the fulfillment of the safety constraint in \eqref{eq:constraint_ZCBF_1} that amounts to having  
$\dot{h}(\xi )\geq  -\alpha(h(\xi))$. As will be shown below in Theorem \ref{thm:1}, by denoting $H(\xi):=L_fh(\xi)+L_gh(\xi){ u_{\text{ref}}(\xi)}+\alpha(h(\xi))$, the optimal solution $u^*$ can be expressed analytically as 
\begin{equation}
u^*(\xi) = \left\{
\begin{array}{ll}
 \sbm{\dot{v}_s \\ \omega_s}  - \frac{L_gh(\xi)^\top}{\left \| L_gh(\xi) \right \| ^2}H(\xi) & \text{if } H(\xi)<0 \text{ and}, \\
 & \left \| L_gh(\xi) \right \|\neq 0;  \\
 \sbm{\dot{v}_s \\ \omega_s} & \text{otherwise} 
\end{array}
\right.
\end{equation}
Note that we use the extended system dynamics \eqref{eq:dynamics} for solving the mixed relative degree issue, while the computed reference control {$u_{\text{ref}}$} and optimal controller $u^*$ are still implementable in the original state equations of the unicycle. One needs to integrate the first control input in $u^*$ to get the longitudinal velocity input $v^*(t)=\int_0^tu^*_1(\tau)\dd\tau$.  

In the following, we will analyze the closed-loop system of the extended plant dynamics in \eqref{eq:dynamics} with the ZCBF-based controller, which is given by the solution of QP problem in \eqref{eq:QP-zcbf}. In particular, we will provide an analytical expression of $u^*$, show the collision avoidance property, and finally present the asymptotic convergence to the source $(x^*,y^*)$.

\begin{theorem}\label{thm:1}
Consider the extended system of unicycle robot in $\eqref{eq:dynamics}$ with globally Lipschitz $f$ and locally Lipschitz $g$. Let the ZCBF $h$ be given as in \eqref{eq:zcbf} with $\mathcal{X}_{s,\text{ext}}$ as in \eqref{eq:safeset_extended} and $u_{\text{ref}}$ be derived from \eqref{eq:SS}. The obstacle-occupied environment is covered by a twice-differentiable, strictly concave source field $J$ which has a unique global maximum at the source location $(x^*,y^*)$.
Then the following properties hold:
\begin{description}
\item[{\bf P1}.] The QP problem \eqref{eq:QP-zcbf}-\eqref{eq:constraint_ZCBF_1} admits a unique solution $u^*(\xi)$ that is locally Lipschitz in  $\mathcal{X}_{s,\text{ext}}$;
\item[{\bf P2}.] The safe set $\mathcal X_{s,\text{ext}}$ is forward invariant, e.g., the state $\xi$ stays in safe set and $\sbm{\xi_1\\ \xi_2}$ avoids collision with obstacles, i.e. \eqref{eq:safe_goal} holds;
\item[{\bf P3}.] The unicycle robot asymptotically converges to the source, i.e. \eqref{eq:convergence-0} holds. 
\end{description}
\end{theorem}
\begin{proof}
We will first prove the property {\bf P1}. Let us rewrite \eqref{eq:QP-zcbf}-\eqref{eq:constraint_ZCBF_1} into the following QP terms of a shifted decision variable {$e= u- u_{\text{ref}}\in\mathbb{R}^{2\times 1}$ with $u_{\text{ref}} =\sbm{a_s\\ \omega_s}=\sbm{\dot{v}_s\\ \omega_s}$ } as 
 \begin{align}
         \qquad & \qquad  \qquad e^* =\mathop{\mathrm{argmin}}\limits_{u \in \mathcal{U}} \frac{1}{2} e^\top e
\\  \label{cons_z0}
          \text{s.t.} \qquad &  L_fh(\xi )+L_gh(\xi )\left(e+\sbm{\dot{v}_s \\ \omega_s}\right)+\alpha(h(\xi ))\geq 0
\end{align}
Correspondingly, define a Lagrangian function $L$ that incorporates the constraint \eqref{cons_z0} by a Lagrange multiplier $\lambda$ as follows
\begin{equation}
\setlength{\belowdisplayskip}{0pt}
    L(e,\lambda) = \frac{1}{2}e^\top e-\lambda\left(L_fh(\xi )+L_gh(\xi )\left(e+\sbm{\dot{v}_s \\ \omega_s}\right)+\alpha(h(\xi ))\right)
\end{equation}
This QP problem for optimal safe control can be solved through the following Karush–Kuhn–Tucker (KKT) optimality conditions
\begin{equation}\label{eq:KKT_zcbf}
    \left.\begin{matrix}
    \begin{aligned}
\frac{\partial L(e^*,\lambda^* )}{\partial e^* } = {e^*}^\top-\lambda^* L_gh(\xi ) = \sbm{0&0} &\\ 
\lambda^*\left(L_fh(\xi )+L_gh(\xi )\left(e^*+\sbm{\dot{v}_s &\\ \omega_s}\right)+\alpha(h(\xi ))\right)=0    &\\ 
 L_fh(\xi )+L_gh\left(\xi )\left(e^*+\sbm{\dot{v}_s &\\ \omega_s}\right)+\alpha(h(\xi )\right)\geq 0  & \\ 
\lambda^* \geq 0
\end{aligned}
\end{matrix}\right\}
\end{equation}
Based on the property of Lagrange multiplier $\lambda^* \geq 0$ and \eqref{eq:KKT_zcbf}, we derived the optimal solution $u^*$ (as well as $e^*$) in the following two cases: when  $\lambda^* > 0$ (i.e., $\lambda^* \neq 0$), or $\lambda^* = 0$. For convenience, let us define 
\begin{equation}\label{eq:H}
 H(\xi) := L_fh(\xi ) + L_gh(\xi ) \sbm{\dot{v}_s \\ \omega_s}+\alpha(h(\xi )).    
\end{equation}

\textbf{\emph{Case-1: $\lambda^* = 0$}}. 
From the first condition in \eqref{eq:KKT_zcbf} and with $\lambda^*=0$, it is clear that $e^* = \lambda^* L_gh(\xi )^\top=  \sbm{ 0 \\ 0}$. By the definition of $e$, this implies that $ \sbm{0\\0} = e^* = u^*-\sbm{\dot{v}_s\\ \omega_s}$, i.e., $u^* = \sbm{\dot{v}_s\\ \omega_s} = u_{\text{ref}}$ holds. By substituting $e^*=\sbm{0\\0}$ to the third condition of \eqref{eq:KKT_zcbf}, we obtain that 
\begin{equation}
      L_fh(\xi )+L_gh(\xi )\sbm{\dot{v}_s \\ \omega_s}+\alpha(h(\xi))\geq 0 
\end{equation}
or in other words, $H(\xi) \geq 0$ holds.

\textbf{\emph{Case-2: $\lambda^* > 0$}}. 
In order to satisfy the second condition in  \eqref{eq:KKT_zcbf} for $\lambda^* > 0$, it follows that 
\begin{equation}\label{eq:con}
     L_fh(\xi )+L_gh(\xi )\left(e^*+\sbm{\dot{v}_s \\ \omega_s}\right)+\alpha(h(\xi)) =0
\end{equation}
By substituting the solution of the first condition in \eqref{eq:KKT_zcbf} with $e^*= \lambda^* L_gh(\xi )^\top$ to the above equation, we get
\begin{equation}
     L_fh(\xi )+L_gh(\xi )\left(\lambda^* L_gh(\xi )^\top+\sbm{\dot{v}_s \\ \omega_s}\right)+\alpha(h(\xi )) =0.
\end{equation}
Accordingly, the Lagrange multiplier $\lambda^*$ can be expressed as
\begin{equation}\label{eq:lambda}
      \lambda^* = -\frac{L_fh(\xi )+L_gh(\xi ) \sbm{\dot{v}_s \\ \omega_s}+\alpha(h(\xi ))}{L_gh(\xi ) L_gh(\xi )^\top},
\end{equation}
where $L_gh(\xi ) L_gh(\xi )^\top$ is scalar.  
Therefore, the optimal $e^*$ and $u^*$ satisfy
\begin{equation}
\begin{aligned}
    e^* &= \lambda^* L_gh(\xi )^\top \\ 
        &= -\frac{L_fh(\xi )+L_gh(\xi ) \sbm{\dot{v}_s \\ \omega_s}+\alpha(h(\xi ))}{L_gh(\xi ) L_gh(\xi )^\top} L_gh(\xi )^\top\\
         &= -\frac{H(\xi)}{\|L_gh(\xi )\|^2} L_gh(\xi )^\top,\\
u^* & = e^* +  \sbm{\dot{v}_s \\ \omega_s} = e^* + u_{\text{ref}}.
\end{aligned} 
\end{equation}
Since we consider the case when $\lambda^* > 0$, it follows from $\eqref{eq:lambda}$ that necessarily 
\begin{equation}
    L_fh(\xi ) + L_gh(\xi ) \sbm{\dot{v}_s \\ \omega_s}+\alpha(h(\xi )) <0.
\end{equation}
In other words, $H(\xi) < 0$. 

As a summary, the closed form of the final pointwise optimal solution $u^*$ and $e^*$ are given by 
\begin{equation}\label{eq:opt_u*}
u^*(\xi) = \left\{
\begin{array}{ll}
 \sbm{\dot{v}_s \\ \omega_s}  - \frac{L_gh(\xi)^\top}{\left \| L_gh(\xi) \right \| ^2}H(\xi) & \text{if } H(\xi)<0 \text{ and}, \\
 & \left \| L_gh(\xi) \right \|\neq 0; \\
 \sbm{\dot{v}_s \\ \omega_s} & \text{otherwise} 
\end{array}
\right.
\end{equation}
\begin{equation}
 e^*(\xi ) = \left\{
\begin{array}{ll}
  - \frac{L_gh(\xi)^\top}{\left \| L_gh(\xi) \right \| ^2}H(\xi) & \text{if } H(\xi)<0 \text{ and}, \\
  & \left \| L_gh(\xi) \right \|\neq 0; \\
  \sbm{0\\0} & \text{otherwise} 
\end{array}
\right.
\end{equation}
where $H(\xi )$ is defined in \eqref{eq:H}.

Let us define the following Lipschitz continuous functions  
\begin{align}\label{eq:omega_1}
    \ell_1(r)&=\left\{\begin{array}{lr}
0 & \forall r\geq0\\ 
r & \forall r<0
\end{array}\right.
\\ \label{eq:omega_2}
   \ell_2(\xi )&= H (\xi ) 
\\ \label{eq:omega_3}
    \ell_3(\xi ) &= - \frac{L_gh(\xi)^\top}{\left \| L_gh(\xi) \right \| ^2},
\end{align}
where $H$ is given in \eqref{eq:H}. 
Since the reference control signal $ u_{\text{ref}}=\sbm{\dot{v}_s \\ \omega_s}$ obtained from source seeking algorithm is Lipschitz continuous and since the $f(\xi)$ and $g(\xi)$ of \eqref{eq:dynamics} are Lipschitz continuous, as well as the derivative of $h(\xi)$ and $\alpha(h(\xi))$, we now have both $L_fh(\xi)$ and $L_gh(\xi)$  Lipschitz continuous in the set {$\mathcal{X}_{s,\text{ext}}$}. Using \eqref{eq:omega_1}--\eqref{eq:omega_3}, we define 
\begin{equation} \label{w}
    e^*(\xi ) = \ell_1(\ell_2(\xi ))\ell_3(\xi ) \qquad  \forall\xi  \in{ \mathcal{X}_{s,\text{ext}}}. 
\end{equation}
Since the function $\ell_3(\xi )$ is  locally Lipschitz continuous in { $\mathcal{X}_{s,\text{ext}}$} (with $\left \| L_gh(\xi) \right \|\neq 0$), the final unique solution $e^*(\xi )$ (or $u^*(\xi ) = e^*(\xi )+u_{\text{ref}}(\xi )$) is locally Lipschitz continuous in $\mathcal{X}_{s,\text{ext}}$. This proves the claim of property {\bf P1}.
 
Let us now prove the property {\bf P2} where we need to show the forward invariant property of the safe set $\mathcal{X}_\text{s,ext}$. Using the given ZCBF $h(\xi)$, we define 
\begin{equation}
    K_{zcbf}(\xi)=\left \{ u \in \mathcal{U}:L_fh(\xi)+L_gh(\xi) u + \alpha (h(\xi) ) \geq  0  \right \}
\end{equation}
for all $\xi \in { \mathcal{X}_{s,\text{ext}}}$. As established that the optimal solution $u^*$ is locally Lipschitz in $\mathcal{X}_{s,\text{ext}}$, the resulting Lipschitz continuous control $u^*(\xi ):{ \mathcal{X}_{s,\text{ext}}}\rightarrow \mathcal{U}$ satisfies $u^*(\xi) \in K_{zcbf}(\xi)$,  renders the safe set $\mathcal{X}_{s,\text{ext}}$ forward invariant \cite[Corollary 2]{Ames_TAC}. In other words, for the system \eqref{eq:dynamics}, if the initial state $\xi(0) \in \mathcal{X}_{\text{s,ext}}$ then $\xi(t)  \in \mathcal{X}_{\text{s,ext}}$ for all $t\geq 0$, i.e. the robot trajectory $\xi $ will remain in the safe set $\mathcal{X}_{\text{s,ext}}$ for all time $t\geq 0$. 

Finally, we proceed with the proof of the property {\bf P3} where \eqref{eq:convergence-0} holds, and we establish this property by showing that the position of the robot will converge only to the source location $(x^*, y^*)$. Recall the signal distribution in the obstacle-occupied environment (in Assumption \ref{ass:J}), consider the case that CBF is not active ($H(\xi)\geq 0$), then the robot is steered by the reference controller and the convergence proof in our work \cite[Proposition III.1]{Li} is applicable. It is established that for positive gains $k_1, k_2$ and for the twice-differentiable, radially unbounded strictly concave function $J$ with maxima at $(x^*, y^*)$, the gradient-ascent controller $u_s=\sbm{v_s\\ \omega_s}$ as in \eqref{eq:SS} {(which corresponds to the reference control signal $u_{\text{ref}} =\sbm{\dot{v}_s \\ \omega_s}$ in \eqref{eq:QP-zcbf})}  guarantees the boundedness and convergence of the closed-loop system state trajectory to the optimal maxima (source) for any initial conditions. 

Subsequently, considering that CBF is active and the robot is driven by the optimal input $u^*$ with respect to the safety constraint, we will show that the mobile robot will not be stationary at any point in $\mathcal X_{s,\text{ext}}$ except at $(x^*, y^*)$. As defined in \eqref{eq:zcbf}, the stationary points belong to the set $\{\xi | \dot h(\xi)=0{\}}$. Indeed, when $\dot h = 0$, $h$ is constant at all time, which implies that either it is stationary at a point $(x,y)$ or it moves along the equipotential lines (isolines) with respect to the closest obstacle $\mathcal O_i$. We will prove that both cases will not be invariant except at $(x^*,y^*)$. 
By substituting the optimal input $u^*(\xi)$ as \eqref{eq:opt_u*} into $\dot{h}(\xi)=L_fh(\xi)+L_gh(\xi)  u^*(\xi)$, we have that
 \begin{equation}\label{eq:dot_h_proof}
 \setlength{\abovedisplayskip}{3pt}
\setlength{\belowdisplayskip}{3pt}
 \dot{h}(\xi )= \left\{
\begin{array}{ll}
-\alpha(h(\xi)) & \text{if } H(\xi)<0 \text{ and}, \\
 & \left \| L_gh(\xi) \right \|\neq 0;  \\
 L_fh(\xi)+L_gh(\xi)  \sbm{\dot{v}_s \\ \omega_s} & \text{otherwise,} 
\end{array}
\right.
\end{equation}
where $H(\xi )$ is defined in \eqref{eq:H}. Given the forward invariance of the safety set $\mathcal X_{s,\text{ext}}$ as proved in {\bf P2}, the trajectory will never touch the open boundary $\partial\mathcal X_{s,\text{ext}}$ where $D=0$ (i.e. $h(\xi)=0$), if $\xi(0)\in \mathcal{X}_{s,\text{ext}}$. Then, in the case of active CBF (i.e., $h(\xi)>0, H(\xi)<0$ and $\left \| L_gh(\xi) \right \|\neq 0$), it is clear that the robot will not remain in $\{\xi | \dot h(\xi)=0\}$, as $\dot h(\xi) = -\alpha(h(\xi)) =0$ iff $h(\xi)=0$ and $\alpha$ is an extended class $\mathcal K$ function. 

Consider the other case that $h(\xi)>0$ and $H(\xi)\geq 0$, $\dot h$ will not be equal to zero at all time as $\xi$ is driven by the source-seeking law in a gradient field with only one unique maximum. This concludes the proof that the robot does not remain stationary except at the source location $(x^*, y^*)$. 
\end{proof}

Alternatively, we propose another collision avoidance approach where the unicycle dynamics \eqref{eq:unicycle_model} is considered as a time-varying affine nonlinear system as follows
\begin{equation} \label{eq:unicycle dynamics}
\setlength{\abovedisplayskip}{5pt}
\setlength{\belowdisplayskip}{5pt}
\dot{\xi}= 
\underbrace{\bbm{
\cos(\xi_3)v(t) \\\sin(\xi_3)v(t)\\ 0}}_{f(\xi,t )}+\underbrace{\bbm{
0 \\ 0\\1 }}_{g(\xi )}\underbrace{\omega}_{u}
\end{equation}
where the angular velocity $\omega$ will be used to avoid the collision with the obstacles and the longitudinal velocity $v=k_1 \left\langle \vec{o}(\theta),\nabla J(x,y) \right\rangle$ is controlled by the source-seeking as in \eqref{eq:SS}. Given the safe set $\mathcal{X}_s$ in \eqref{eq:safeset}, a reciprocal control barrier function can be constructed as
$  B(\xi) = \frac{1}{D(x,y)e^{P(x,y,\theta)}}  $
with the distance $ D(x,y)$ and orientation function 
 $P(x,y,\theta) = (\theta - \beta)\delta$,
where $\beta= \text{arctan}\frac{y_{\text{obs},i}-y}{x_{\text{obs},i}-x}$ is the bearing angle between the robot and the closest obstacle's boundary, and the parameter $\delta\in\mathbb{R}_+$ is a control parameter. The minimum safe distance $d_{\text{safe}}$ is in the function $D(x,y)$. Following \cite[Corollary 1]{Ames_TAC}, given the RCBF $B(\xi)$, for all  $\xi  \in \text{Int}(\mathcal{X}_s)$ satisfying \eqref{eq:safeset0}-\eqref{eq:safeset2}, the forward invariance of $\mathcal X_s$ is guaranteed if the locally Lipschitz continuous controller $u$ satisfies
   $ L_fB(\xi ) +L_gB(\xi )u-\alpha_3(h(\xi )) \leq 0$
where $\alpha_3$ is a class $\mathcal{K}$ function, and $h(\xi) = \frac{1}{B(\xi)}$. Then the RCBF-based safety-guaranteed source-seeking problem can be formulated as a quadratic programming (QP) below
\begin{align}\label{eq:QP-rcbf}
    &u^* =\mathop{\mathrm{argmin}}\limits_{u \in \mathcal{U}} \left \| u - \omega_s \right \|^2 \\
\nonumber     \text{s.t.}  \quad & L_fB(\xi )+L_gB(\xi )u-\alpha_3(h(\xi ))\leq 0 
\end{align}
Note that the unicycle robot's motion variables (i.e. velocity $v$ and angular velocity $\omega$) are controlled independently, the sole use of angular velocity for collision avoidance will no longer pose a mixed relative degree problem. Besides, as the linear velocity $v$ is controlled for source-seeking, this approach does not limit the control performance as the general methods where $v$ is fixed \cite{xiao2021high}.

As a remark, considering the practicality and implementation, an additional control input constraint can be introduced to the above CBF-QP architecture such that the control signals are bounded within an admissible set $\mathcal{U}_{\text{adm}} =\left \{ {u}=\sbm{{u}_{1} \\ {u}_{2}}\in\mathbb{R}^2 \mid a_{\text{min}}\leq {u}_{1} \leq a_{\text{max}},  \omega_{\text{min}}\leq {u}_{2} \leq \omega_{\text{max}}\right \}$. More details can be found in our follow-up work \cite{Li_connectivity}. 

\vspace{-0.1cm}
\section{Simulation setup and results} \label{sec:simulation}
In this section, we validate the proposed method by using numerical simulations in Matlab and in the Gazebo/ROS environment. Firstly, Monte Carlo simulations are conducted in a Matlab environment to validate the methods where static obstacles are considered. Subsequently, we evaluate the efficacy and generality of the approach when it has to deal with a realistic environment with dynamic obstacles by running simulations in Gazebo/ROS where multiple walking people are considered. In all Matlab simulations, the stationary source location $(x^*,y^*)$ is set at the origin, while in the Gazebo/ROS simulations, the source is randomly positioned. 

\vspace{-0.2cm}
\subsection{Simulation Results in Matlab}
 For the collision avoidance in the static environment, we consider multiple circular-shaped obstacles with different radii $r\in\{0.7\text{m}, 0.8\text{m}, 0.9\text{m}, 1.0\text{m}, 1.2\text{m}\}$ that are set randomly around the source, and the minimum distance between obstacles is $d_{\text{min}} = 0.8\text{m}$ in this environment. The extended safe set on the extended state space $\Xi$ is given by \eqref{eq:safeset_extended} where $\mathcal{O}_{i}$ denotes the boundary of the $i_{th}$ obstacle. The robot is required to maintain a safe margin $d_{\text{safe}} = 0.1\text{m}$ to the obstacles' boundary. A simple quadratic concave field is distributed as $J(x,y)= -\sbm{x&y}H\sbm{x\\y}$ where $H=H^T>0$. It has a unique maximum at the origin and its local gradient vector is given by $\nabla J(x,y) = - 2\bbm{x&y}H$. The source gradient and distance to the obstacles are observed by the robot in real-time for the computation of ZCBF-based 
 control input.
 
As the first control input in \eqref{eq:dynamics} refers to the longitudinal acceleration, instead of the usual longitudinal velocity in the unicycle model, we use the following Euler approximation to get the longitudinal velocity input from the computed longitudinal acceleration input $a(t)$ in the discrete-time numerical simulation as $ v_{s}(t_{k+1}) = v_{s}(t_{k}) + d_t a(t_k) $ where $t_k$ and $t_{k+1}$ denote the current and next discrete-time step, respectively, and $d_t$ is the integration time step. 

In the first simulation, we consider the ZCBF in \eqref{eq:zcbf} with an example $D(x,y)$ given as below
\begin{equation}\label{eq:D_examp}
  D(x,y) = \left\{
\begin{array}{ll}
e^{-\frac{1}{\gamma d_{\text{cons}}}} - e^{\frac{1}{\gamma(d_{\text{ro}}(x,y)-d_{\text{cons}})}}, \quad & d_{\text{ro}}(x,y)< d_{\text{cons}}\\ 
e^{-\frac{1}{\gamma d_{\text{cons}}}}, \quad & d_{\text{ro}}(x,y) \geq d_{\text{cons}}
\end{array}\right.
\end{equation}
where $\gamma$ is a positive multiplier.
Figure \ref{mat:zcbf} shows two simulation results using the same $[0,10]\times [0,10]$ environment with $9$ circular obstacles.
The color gradient shows the source field where Sub-figure \ref{fig:b} is based on $H=I$ and \ref{fig:c} is based on $H=\sbm{5 & 4\\4 & 5}$. The robot is initialized in random initial conditions and the parameters of the reference source seeking input control are set as $k_1=1, k_2=5$ in sub-figure \ref{fig:b}, and $k_1= 0.5, k_2= 5$ for another case. 
These results show that the unicycle can maneuver around the obstacles of different dimensions, guarantee a safe margin $d_{\text{safe}} = 0.1\text{m}$ and converge to the maximum point $(0,0)$, as expected. Particularly, Figure \ref{fig:state-time} shows the evolution of robot with the initial states as in Figure \ref{fig:b}. 
\begin{figure}[htbp]
	\centering
\vspace{-15pt}
\subfigcapskip=-5pt 
	\subfigure[]{
		\begin{minipage}[t]{0.22\textwidth}
			\centering			
   \includegraphics[width=1\textwidth]{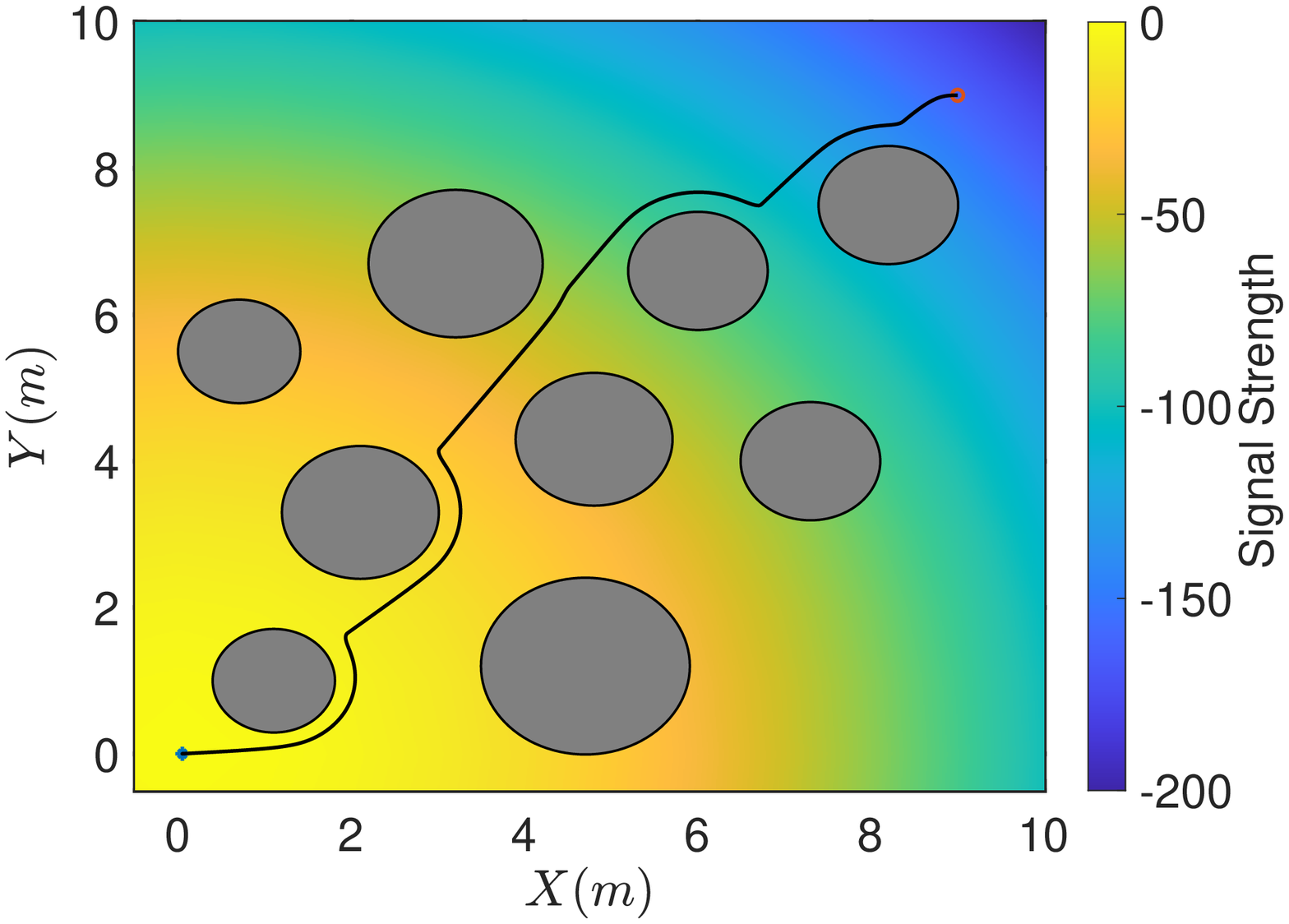}
			\label{fig:b}
		\end{minipage}%
	}%
        \subfigure[ ]{
		\begin{minipage}[t]{0.22\textwidth}
			\centering			\includegraphics[width=1\textwidth]{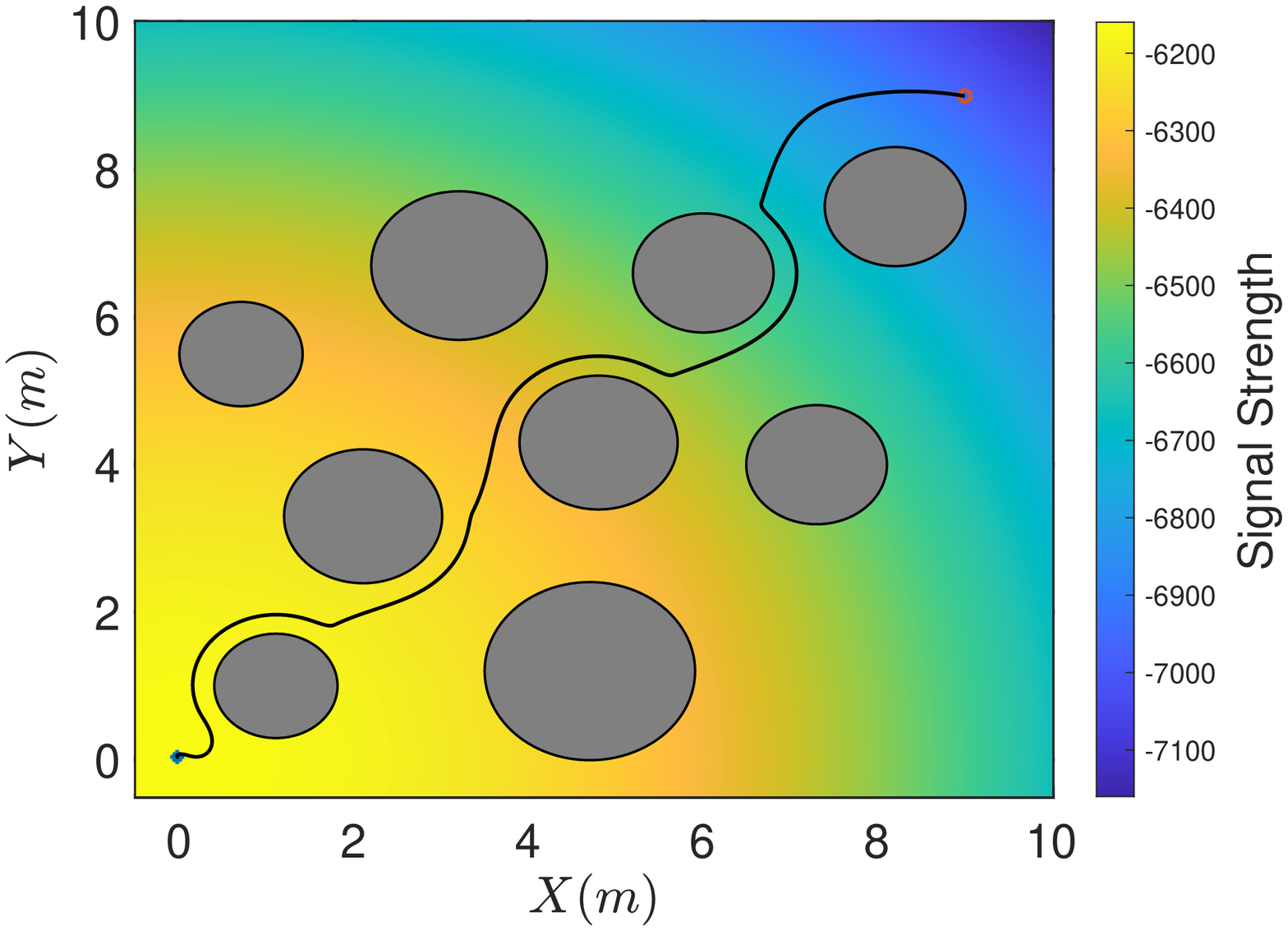}
			\label{fig:c}
		\end{minipage}%
	}%

	\centering
	\caption{Simulation results based on the zeroing control barrier function $ h(x,y,v,\dot{x},\dot{y}) = D(x,y)  e^{-P(x,y,v,\dot{x},\dot{y})}$, where both the longitudinal and angular velocity are given by the optimal solution of ZCBF-QP \eqref{eq:QP-zcbf}. The source is set at the origin $(0,0)$, surrounded by multiple circular-shape obstacles, and the signal strength is distributed in the field: $(a):J_1(x,y)=-x^2-y^2$; $(b):J_2(x,y)= -5x^2-8xy-5y^2$, respectively. The robot is set at initial position (given by $\circ$) to search the source while avoiding any potential collisions.}
	\label{mat:zcbf}
\end{figure}
\begin{figure}[htbp]
	\centering
\vspace{-20pt}
    \subfigure[]{
		\begin{minipage}[t]{0.22\textwidth}
			\centering			
                \includegraphics[width=1\textwidth]{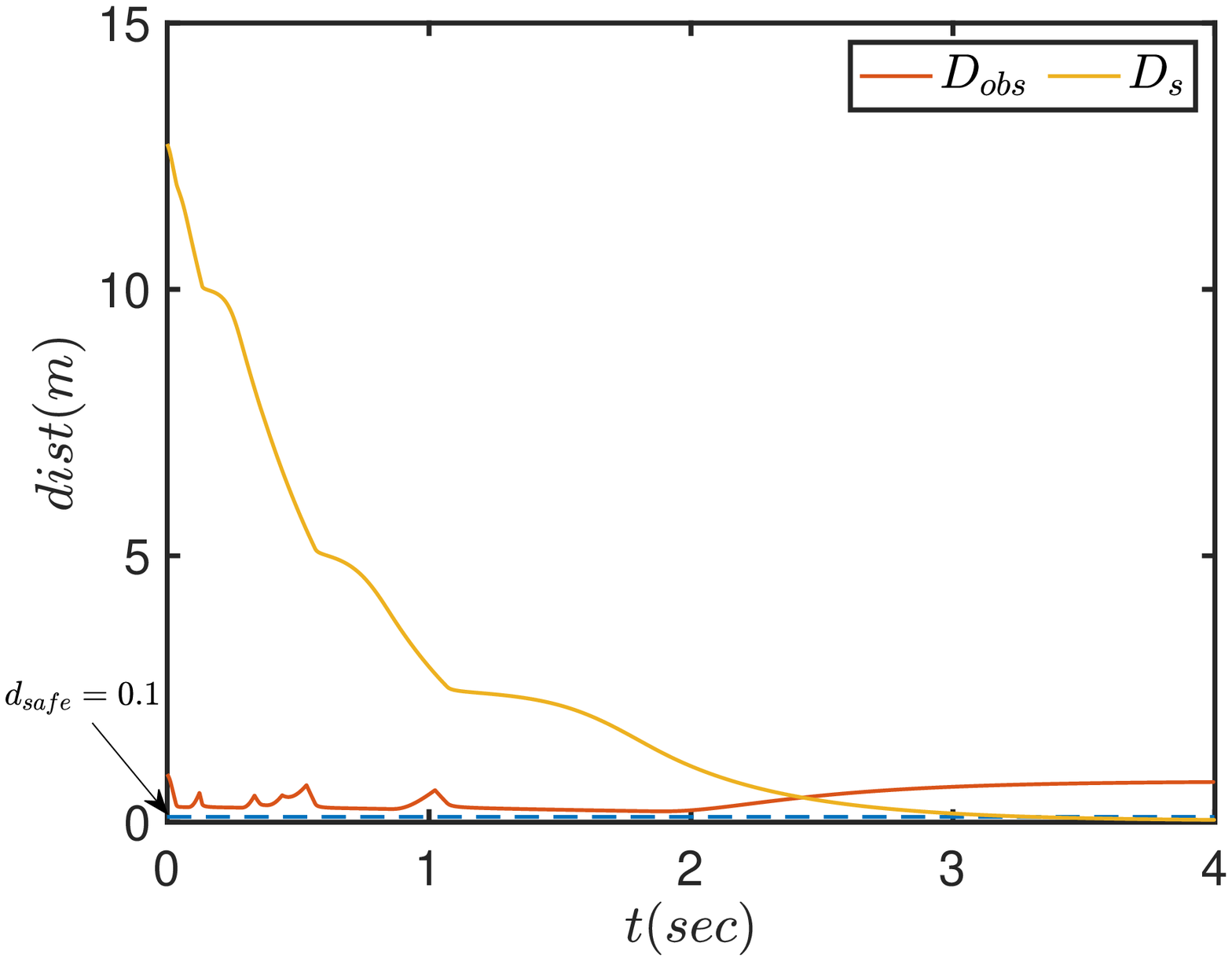}
		\end{minipage}%
	}%
	\subfigure[]{
		\begin{minipage}[t]{0.22\textwidth}
			\centering			
                \includegraphics[width=1\textwidth]{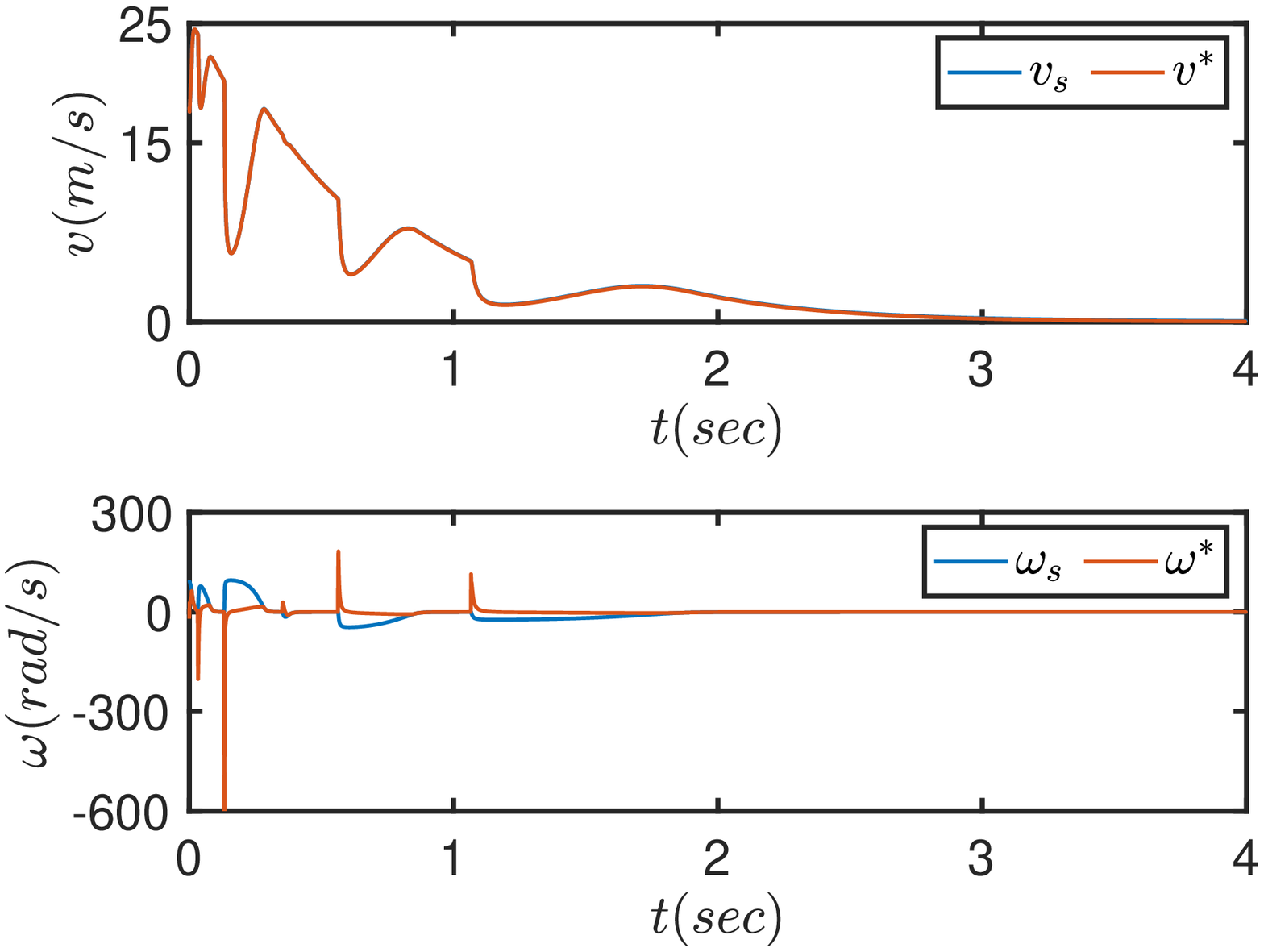}
		\end{minipage}%
	}%
	
	\centering
	\caption{(a). The distance of the robot to the closest obstacle $D_{\text{obs}}= \left\|\sbm{x_{\text{obs}} \\ y_{\text{obs}}}-\sbm{x\\y}\right\| $ and to the source $D_s = \left\|\sbm{x^* \\ y^*}-\sbm{x\\y}\right\|$; (b) The comparison of motion variables where $v_s, \omega_s$ are the nominal source-seeking signals and $v^*, \omega^*$ are the optimized safe control inputs for obstacle avoidance. }
	\label{fig:state-time}
\end{figure}

As a supplementary, we compare the performance of the proposed ZCBF-based collision-free source-seeking with two other ECBF/RCBF-based methods (the details can be found in our supplementary work \cite{Li_CBF}), by using the same reference source-seeking control parameters. In this simulation, instead of the smooth $D(x,y)$ as in \eqref{eq:D_examp}, we deploy a standard distance function $D(x,y) = d_{\text{ro}} $
for verifying the efficacy of our approach in general implementation.
For each method, we run a Monte Carlo simulation of $50$ runs, where on each run, the three methods start from the same randomized initial conditions. The resulting trajectories of the closed-loop systems are recorded and analyzed for comparing the methods. In Figure \ref{fig:comparison}, we present the box plot\footnote{The box plot provides a summary of a given dataset containing the sample median, the $1^{\text{st}}$ and $3^{\text{rd}}$ quartiles, $1.5$ interquantile range from the $1^{\text{st}}$ and $3^{\text{rd}}$ quartiles and the outliers.} drawn from the samples of convergence time and minimum distance to the boundary of the obstacle.  The variable $T_c$ in Figure \ref{fig:comparison}(a) refers to the convergence time for the robot to reach the final $20\%$ of the distance between the source and its initial position. The variable $D_{\text{obs}}$ in Figure \ref{fig:comparison}(b) represents the closest distance between the robot and the obstacle's boundary during the maneuver toward the source point. 

The Monte Carlo simulation results validate the convergence analysis of the proposed collision-free source-seeking control algorithms where the robot remains safe for all time and maintains a minimum safety margin. It is noted that both linear and angular velocities are imposed with safety constraints in ZCBF for collision avoidance, while only the angular velocity is constrained in RCBF and ECBF. In this case, while the methods perform equally well in avoiding the obstacles, the ZCBF-based method outperforms the other two approaches in both the convergence time as well as in ensuring that the safe margin of $d_{\text{safe}} = 0.1m$ from the obstacles' boundary does not trespass. We remark that the ECBF-based and RCBF-based methods may still enter the safe margin briefly due to the discrete-time implementation of the algorithms. This robustness of the algorithms with respect to the disturbance introduced by time-discretization shows a property of input-to-state safety of the closed-loop systems as studied recently in \cite{Romdlony2016b,Romdlony2019}.
\begin{figure}[htbp]
	\centering
\vspace{-10pt} 
\subfigtopskip=0pt 
\subfigbottomskip=0pt 
\subfigcapskip=0pt 
    \subfigure[]{
		\begin{minipage}[t]{0.22\textwidth}
		    \label{fig:ts}
			\centering			\includegraphics[width=1\textwidth]{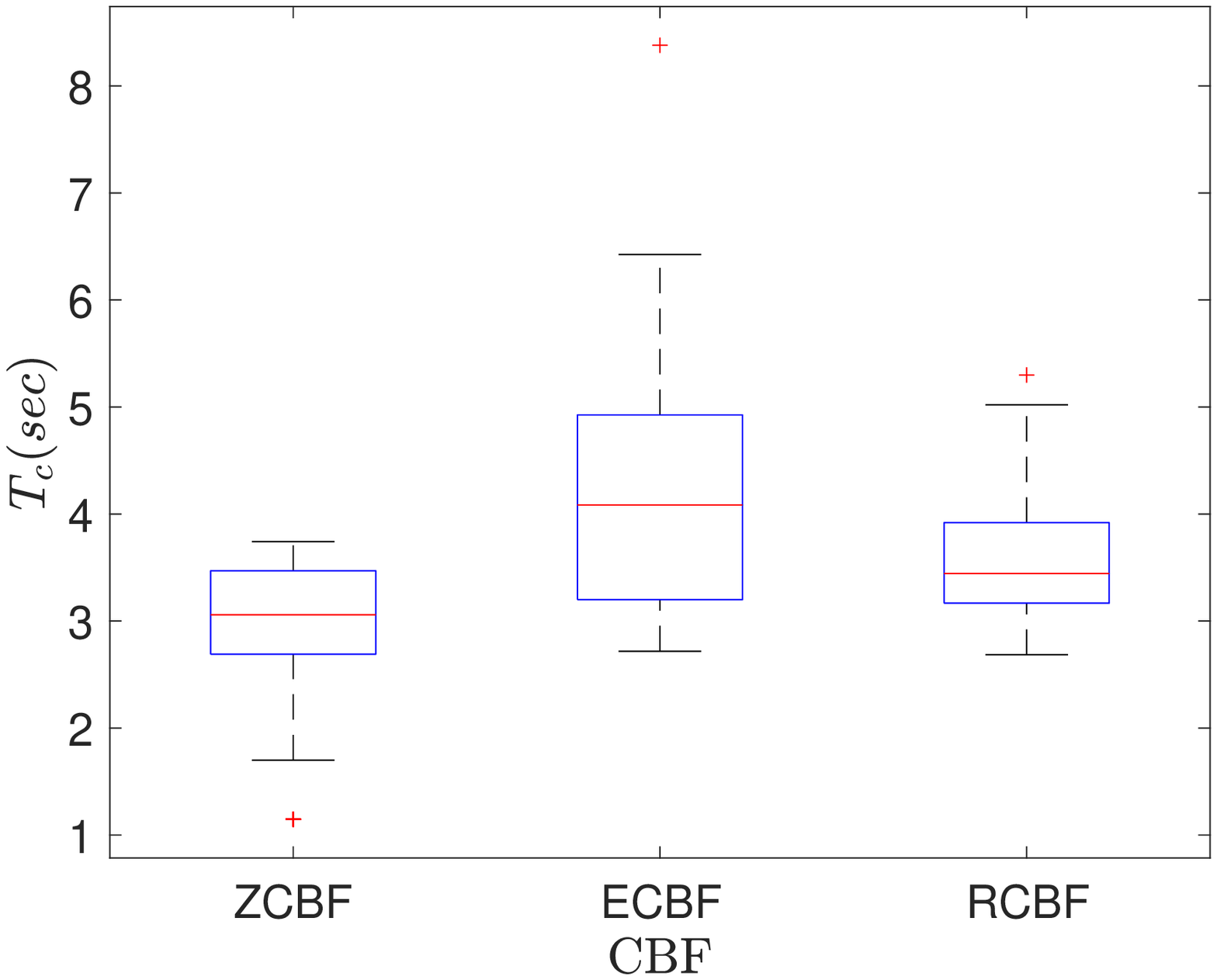}
		\end{minipage}%
	}%
	\subfigure[]{
		\begin{minipage}[t]{0.22\textwidth}
		 \label{fig:Dc}
			\centering			\includegraphics[width=1\textwidth]{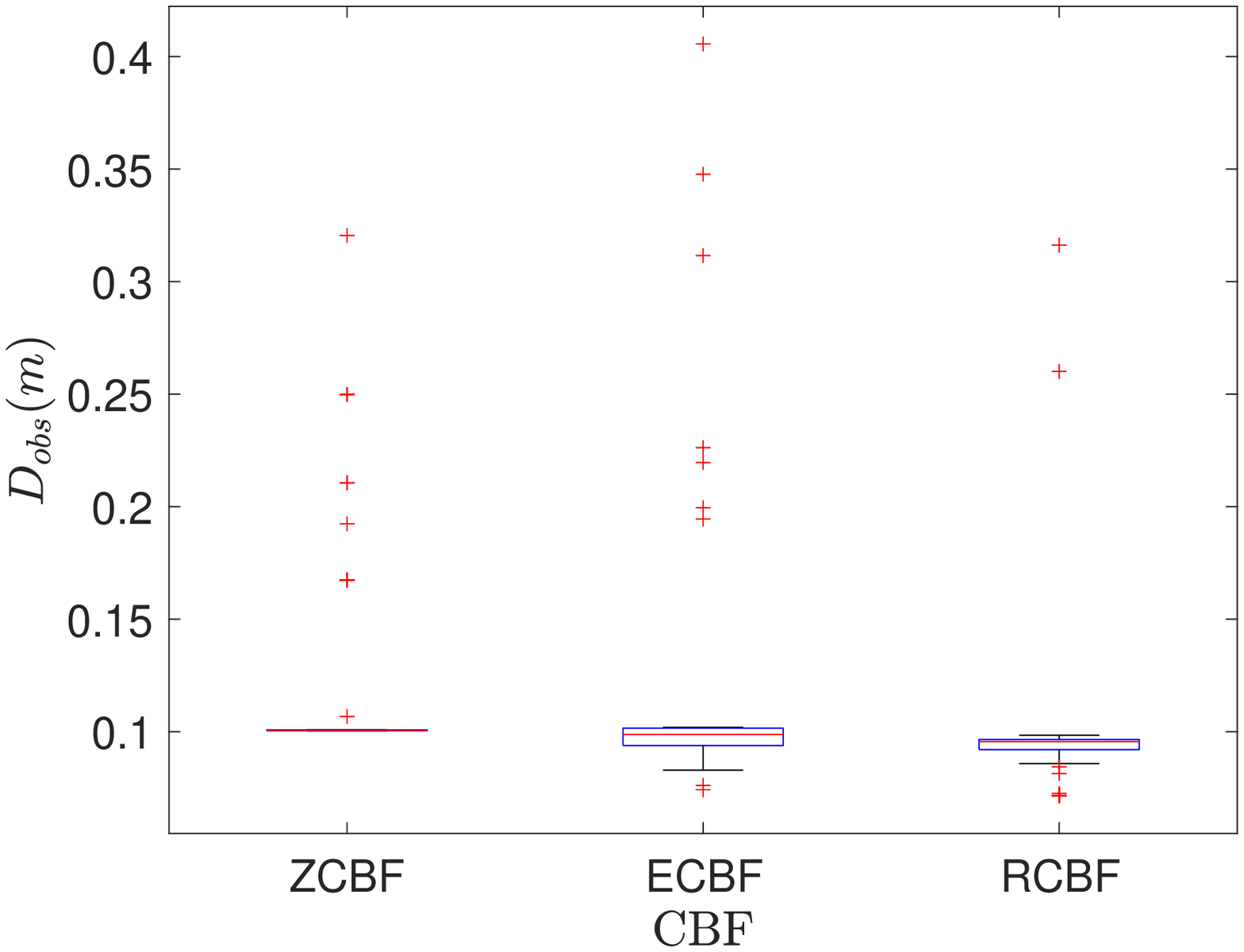}
		\end{minipage}%
	}%

	\centering
     \caption{Box plot of the Monte Carlo simulation of the three control barrier function-based methods using the same reference source seeking control parameters ($k_1= 0.3, k_2 = 30$). The robot is initialized at $50$ random initial positions for each group of simulations. }
	\label{fig:comparison}
\end{figure}

\vspace{-0.5cm}
\subsection{Simulation Results in Gazebo/ROS}
The realistic Gazebo/ROS (Noetic) environment is built and run in Linux (Ubuntu 20.04). Figure \ref{fig:gazebo_world} shows the common 3D room setup where three people walk at different speeds of $0.2 \text{m/s}, 0.4 \text{m/s}, 1 \text{m/s}$, and they are considered as dynamic obstacles by the mobile robot placed at the center of the room. The source field is given by a quadratic field which is illustrated as colored rings and the center for the source field is set randomly in every simulation. We note that the source distribution model is not known apriori in practical applications, and the gradient $\nabla J(x,y)$ can be estimated by local sensor systems.

The unicycle robot is simulated by the turtlebot3 (Waffle) embedded with a laser distance sensor of LDS-01 for detecting all obstacles within the neighborhood of $0.12\sim3.5\text{m}$. Accordingly, the bearing angle to the closest obstacle can be calculated from the $720$ samples of $360\degree$ field-of-view. Considering the dynamic environment due to the movement of pedestrians, the minimum safe distance between the robot and any potential collisions (walls and pedestrians) is set to be $d_{\text{safe}}=0.3\text{m}$ in the safe set $\mathcal{X}_{s,\text{ext}}$ in \eqref{eq:safeset_extended}. We emphasize that only local measurements (e.g., local signal gradient, distance, and bearing angle between the robot and obstacles) are required in the implementation. In particular, the global information, such as the source field function, the location of the source, or the obstacles) is not relayed to the robot.

Figure \ref{fig:gazebo_zcbf} shows the plots of the initial and final states of the robot for the safe source-seeking motion. Given the analysis regarding the static obstacles in Section \ref{sec:control_design}, the robot can safely navigate within the fixed maze and successfully reach the source. Particularly, with high-frequency measuring and computing updates, the proposed method can provide safe performance regarding dynamic obstacles (walking pedestrians). The animated motion of the robot in Gazebo/ROS performing the source seeking while avoiding these dynamic obstacles can be found in the accompanying video at.\footnote{[Online]. Available:
\href{https://youtu.be/D5zXVeOPy30}{https://youtu.be/D5zXVeOPy30} } 
\begin{figure}[htbp]
\centering
\vspace{-0.2cm}
\subfigtopskip=0pt 
\subfigbottomskip=0pt 
\subfigcapskip=-1pt 
\subfigure[]{
    \begin{minipage}[t]{0.22\textwidth}
        \centering			\includegraphics[width=1\textwidth]{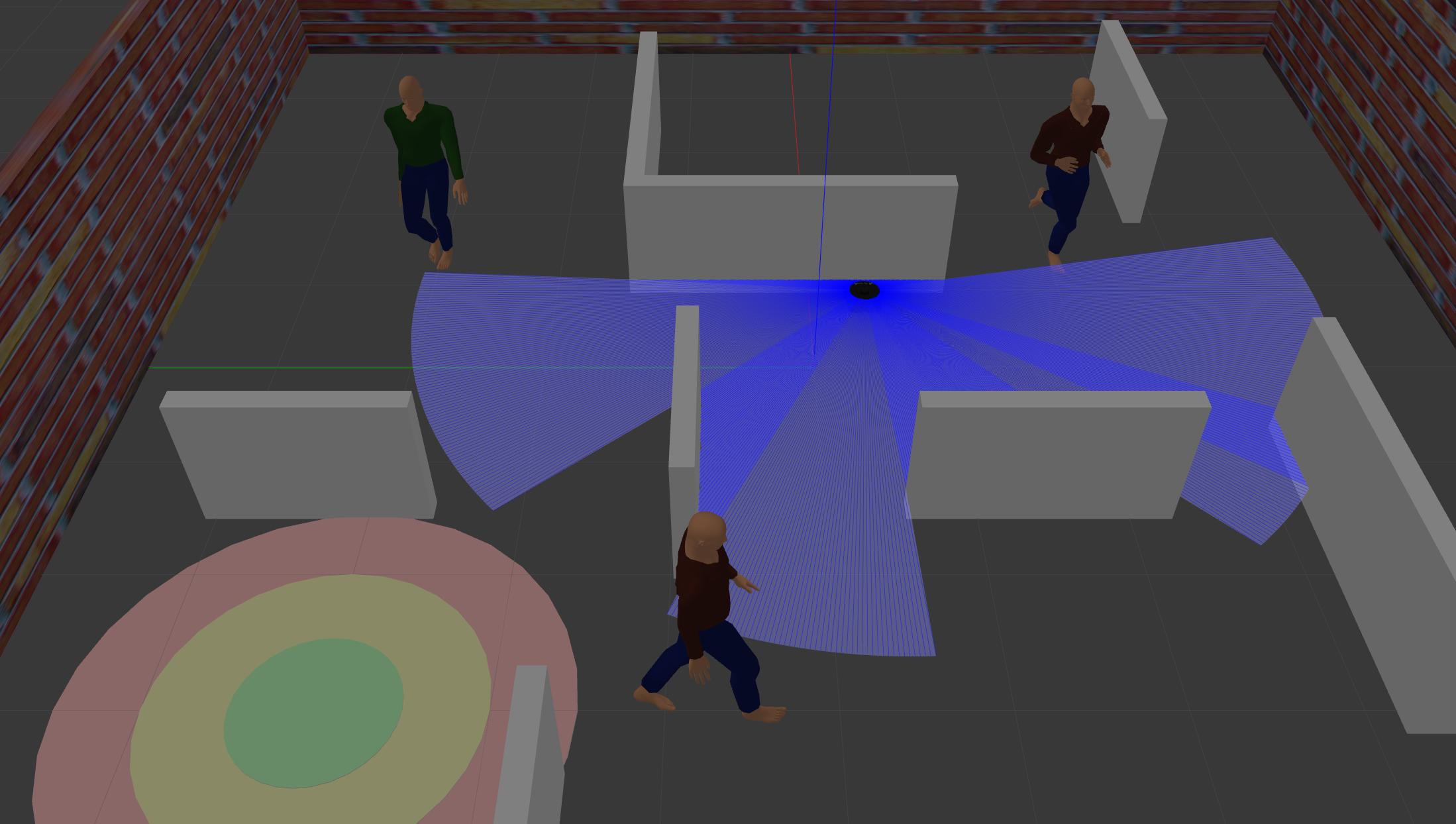}
    \end{minipage}%
}%
\subfigure[]{
    \begin{minipage}[t]{0.22\textwidth}
        \centering			\includegraphics[width=1\textwidth]{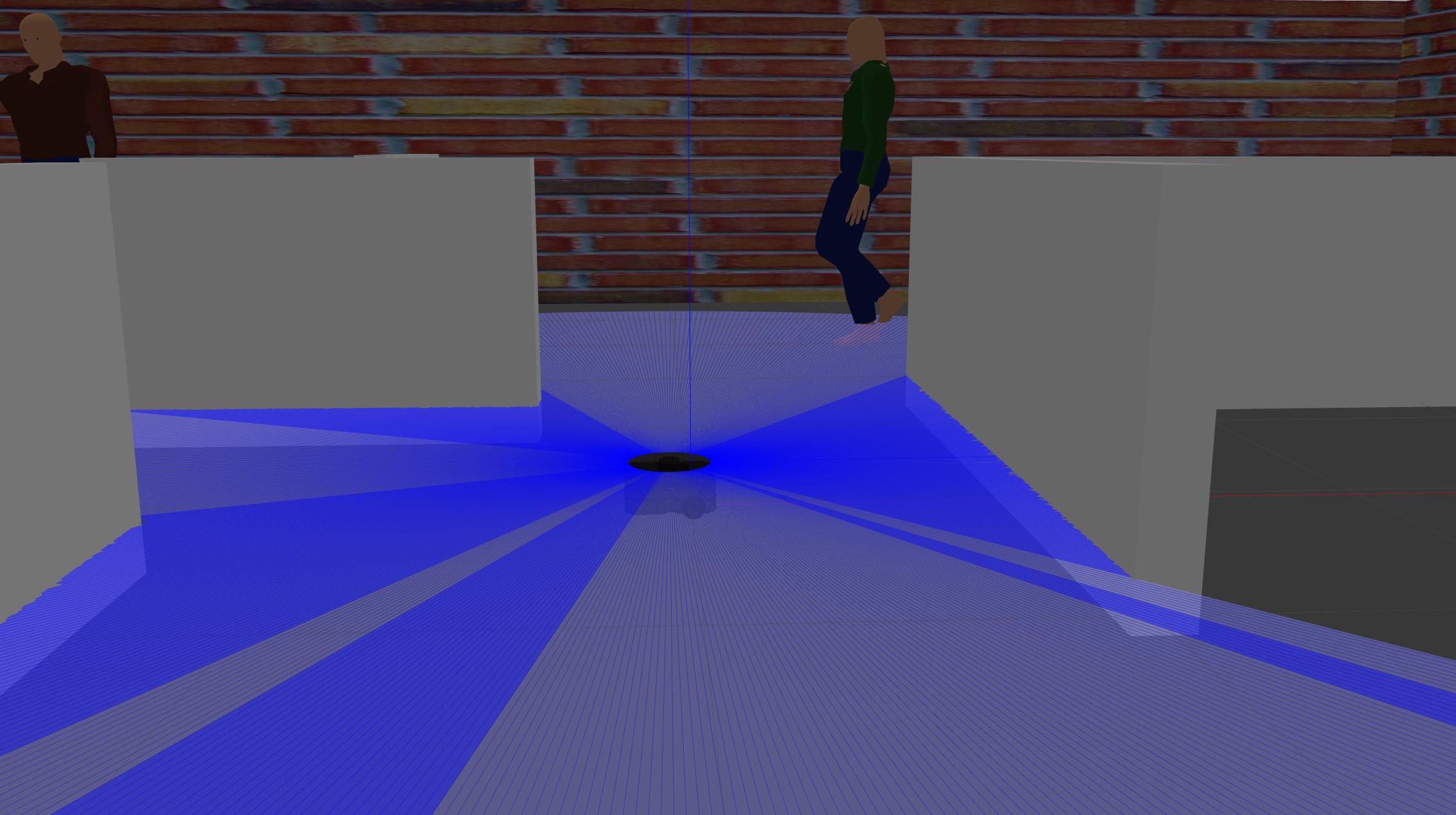}
    \end{minipage}%
}%

\centering
\caption{Laser-embedded turtlebot3 robot in the constructed gazebo world, which contains fixed walls and walking people as potential collisions for robot navigation. The center point of the green circle area stands for the source location.  }
\label{fig:gazebo_world}
\end{figure}

\vspace{-0.2cm}
\begin{figure}[htbp]
\centering
\vspace{-10pt}
\subfigtopskip=0pt 
\subfigbottomskip=0pt 
\subfigcapskip=2pt 
\subfigure[]{
    \begin{minipage}[t]{0.20\textwidth}
        \label{fig:zcbf_left}
        \centering			\includegraphics[width=1\textwidth]{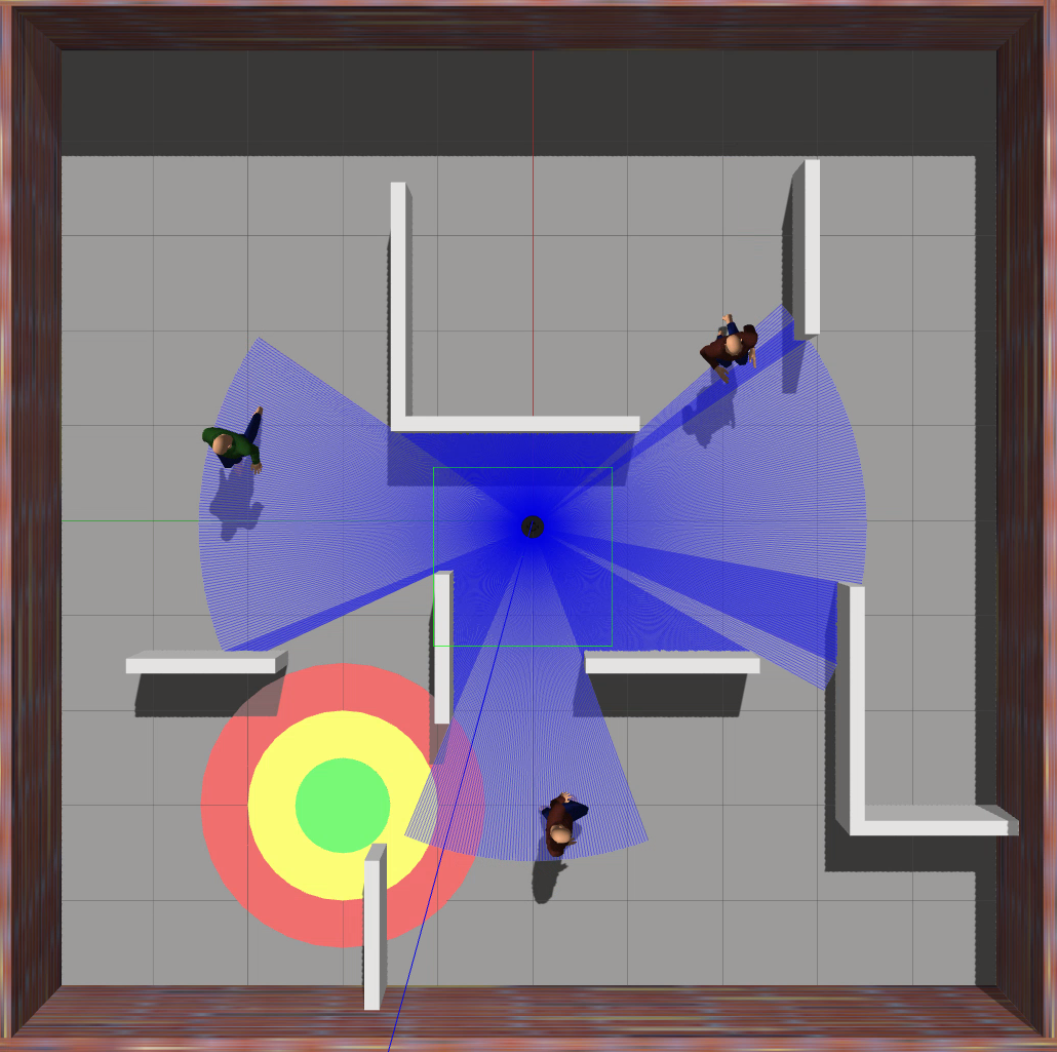}
    \end{minipage}%
}%
\subfigure[]{
    \begin{minipage}[t]{0.20\textwidth}
     \label{fig:zcbf_right}
        \centering			\includegraphics[width=1\textwidth]{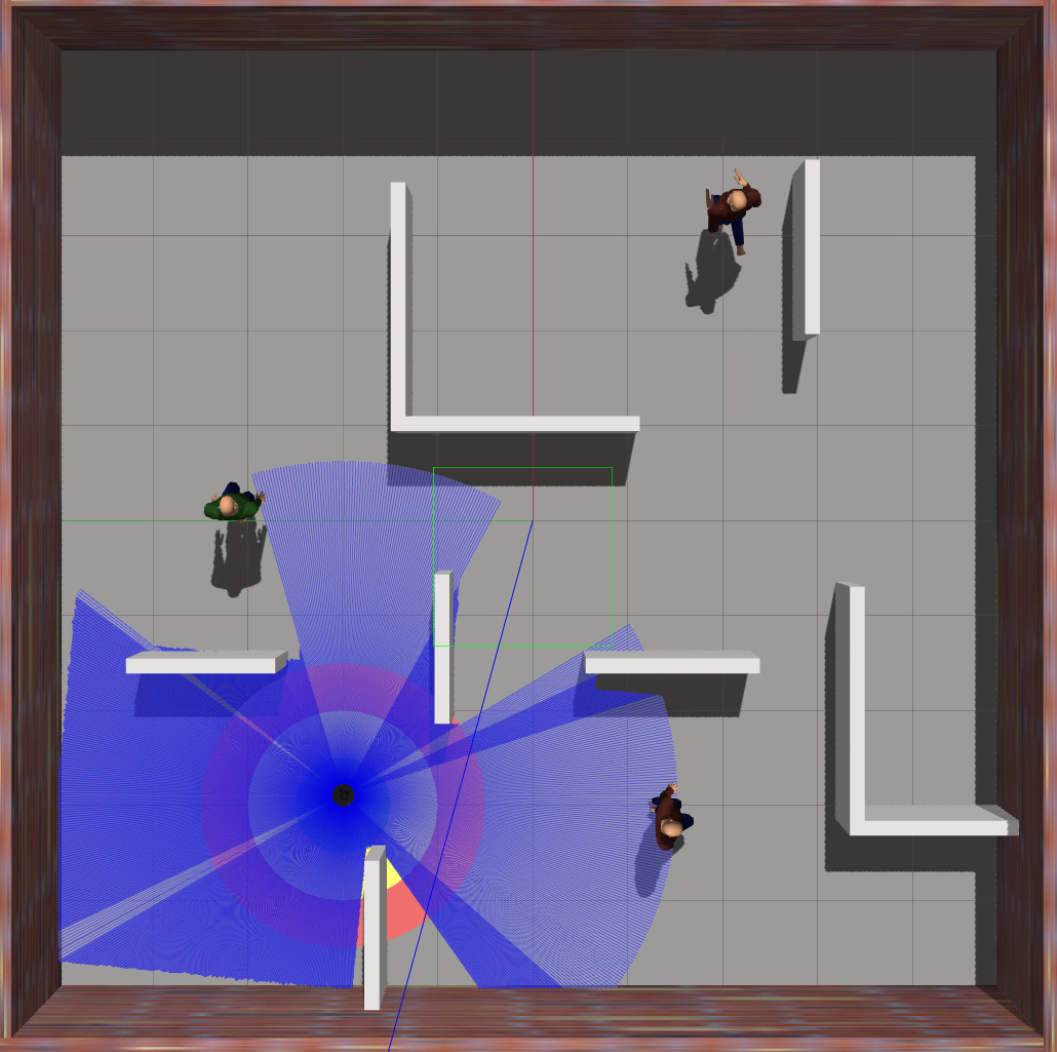}
    \end{minipage}%
}%

\centering
\caption{Simulation of the zeroing control barrier function (ZCBF)-based collision-free source seeking in the unknown environment. Three pedestrians are walking around as dynamic obstacles. Robot is initialized with the state $(0, 0, \frac{1}{2}, \frac{\sqrt{2}}{4}, \frac{\sqrt{2}}{4} )$ in \ref{fig:zcbf_left}, and eventually arrived at the source $(-3,2)$ in \ref{fig:zcbf_right}. The source seeking control parameters in the field $J(x,y) = -(x+3)^2-(y-2)^2$ are set as $k_1=0.2, k_2=0.5$.  }
\label{fig:gazebo_zcbf}
\end{figure}
\vspace{-0.5cm}

\section{Conclusion}\label{sec:conclusion}
In this paper, we presented a framework of the safety-guaranteed autonomous source seeking control for a nonholonomic unicycle robot in an unknown cluttered environment. A construction of the zeroing control barrier function is proposed to tackle the mixed relative degree problem due to the unicycle's nonholonomic constraint. Guided by the projected gradient-ascent control law as the reference control signal, we show that the proposed control barrier function-based QP method gives locally Lipschitz optimal safe source seeking control inputs. The analysis of safe navigation and the convergence to the source are presented. The efficacy of the methods is shown via Monte Carlo simulations in Matlab, as well as, numerical simulations in Gazebo/ROS which provides a realistic environment.
\vspace{-0.1cm}

\vspace{-0.3cm}
\bibliographystyle{IEEEtran}
\bibliography{paper.bib}{}

\end{document}